\documentclass[conference]{IEEEtran}
\IEEEoverridecommandlockouts
\usepackage{cite}
\usepackage{amsmath,amssymb,amsfonts,amsthm}
\usepackage{algorithmic}
\usepackage{graphicx}
\usepackage{textcomp}
\usepackage{xcolor}
\def\BibTeX{{\rm B\kern-.05em{\sc i\kern-.025em b}\kern-.08em
    T\kern-.1667em\lower.7ex\hbox{E}\kern-.125emX}}
    
\newcommand{\be}{\begin{equation}}
\newcommand{\ee}{\end{equation}}
\newcommand{\sset}[1]{\left\{{#1}\right\}}

\newtheorem{thm}{Theorem}

\begin{document}

\title{Sharp Frequency Bounds for Sample-Based Queries}

\author{\IEEEauthorblockN{Eric Bax}
\IEEEauthorblockA{\textit{Verizon Media} \\
Los Angeles, CA \\
ebax@verizonmedia.com}
\and
\IEEEauthorblockN{John Donald}
\IEEEauthorblockA{\textit{Verizon Media} \\
Los Angeles, CA \\
jdonald@verizonmedia.com}
}

\IEEEpubid{\makebox[\columnwidth]{978-1-7281-0858-2/19/\$31.00~\copyright2019 IEEE \hfill} \hspace{\columnsep}\makebox[\columnwidth]{ }}

\maketitle

\IEEEpubidadjcol

\begin{abstract}
A data sketch algorithm scans a big data set, collecting a small amount of data -- the sketch, which can be used to statistically infer properties of the big data set. Some data sketch algorithms take a fixed-size random sample of a big data set, and use that sample to infer frequencies of items that meet various criteria in the big data set. This paper shows how to statistically infer probably approximately correct (PAC) bounds for those frequencies, efficiently, and precisely enough that the frequency bounds are either sharp or off by only one, which is the best possible result without exact computation.
\end{abstract}

\begin{IEEEkeywords}
big data, sampling, statistics
\end{IEEEkeywords} 

\section{Introduction}
Some distributed database and data sketch algorithms take a fixed-size random sample of a big data set \cite{fan62,vitter84,vitter85,olken86,olken90}, and use that sample to infer estimated frequencies of items that meet various criteria in the big data set. This paper shows how to compute probably approximately correct (PAC) bounds for those frequencies. Such bounds can give approximate query results \cite{olken93,acharya00}, can communicate a range of likely answers while a query is still executing \cite{haas96a,haas96b,hellerstein97}, can be a basis for query planning over very large datasets \cite{olken93}, and can help secure database information \cite{denning80}. Direct computation gives bounds that are sharp (given precise tail probability computation, and within one otherwise), in contrast to previous tail bounds \cite{hoeffding63,chvatal79}, which are useful for proofs because they are smooth and were necessary in practice when computers were less powerful. 

\section{Frequency Bounds from Samples}
Let $n$ be the number of items in a big data set. Let $S$ be a random size-$s$ sample of the items, drawn uniformly at random without replacement. Let $k$ be the number of items in $S$ that meet some condition. Let $m$ be the (unknown) number of items in the big data set that meet the condition. Knowing $n$, $s$, and $k$, we want to infer probably approximately correct (PAC) bounds for $m$. Let $\delta$ be the maximum bound failure probability that we are willing to accept. (For brevity, we state results here without proof, and we use the conventions that ${{i}\choose{j}} = 0$ if $j < 0$, $j > i$, or $i < 0$, ${{0}\choose{0}} = 1$, and $0! = 1$.)

Note that the probability that $k$ of $s$ samples meet the condition, given that $m$ of $n$ big data items meet the condition, has a hypergeometric distribution:
$$ p(n, m, s, k) \equiv {{n}\choose{s}}^{-1} {{m}\choose{k}} {{n - m}\choose{s - k}}. $$
The left tail is the probability that $k$ or fewer samples meet the condition:
$$ L(n, m, s, k) \equiv \sum_{i=0}^{k} p(n, m, s, i). $$
The right tail is the probability that $k$ or more samples meet the condition:
$$ R(n, m, s, k) \equiv \sum_{i = k}^{\min(s, m)} p(n, m, s, i). $$
Then an upper bound for $m$ is 
$$ m_u(n, s, k, \delta) \equiv \max \sset{m | L(n, m, s, k) \geq \delta}, $$
and a lower bound is
$$ m_d(n, s, k, \delta) \equiv \min \sset{m | R(n, m, s, k) \geq \delta}. $$

Each bound has failure probability $\delta$. For $1 - \delta$ confidence that upper and lower bounds both hold, use $\frac{\delta}{2}$ in place of $\delta$ in each bound. For $1 - \delta$ confidence that upper and lower bounds hold simultaneously for $j$ different conditions, use $\frac{\delta}{2j}$. (In the worst case, failures are exclusive, so failure rates sum.) 

\section{Computational Challenge}
For simplicity, focus on the upper bound, $m_u$. To apply the results to the lower bound, note that $ m_d(n, s, k, \delta) = n - m_u(n, s, s - k, \delta).$ One goal is to compute $m_u$ to within one. This is the best possible without exact computation, because if $L(n, m_u, s, k) = \delta$, then any arbitrarily small negative error in computing $L(n, m_u, s, k)$ will disqualify $m_u$ from consideration for the bound. The other goal is to compute the bound in a reasonable time. 

Some notation: let $\hat{L}(n, m, s, k)$ be a computed estimate of $L(n, m_u, s, k)$, and let $\hat{m}$ be a computed estimate of $m_u$. The left tail $L(n, m, s, k)$ is positive and strictly decreasing in $m$ for $0 \leq m \leq n - (s - k)$ and zero for $m > n - (s - k)$. So we can use binary search to compute a $\hat{m}$ value such that $\hat{L}(n, \hat{m}, s, k) \geq \delta$ and $\hat{L}(n, \hat{m} + 1, s, k) < \delta$: 
\begin{itemize}
\item Start with low and high $m$ values set to 0 and $n$, respectively. 
\item Let $m_m$ be an integer between the low and high values (near their average).
\item If $\hat{L}(n, m_m, s, k) \geq \delta$, assign $m_m$ to low. Else assign $m_m$ to high. 
\item Repeat the last two steps until low and high are successive integers. 
\item Return low as $\hat{m}$.
\end{itemize}

The starting low and high values need not be 0 and $n$ -- they can be any values such that the left tail is at least $\delta$ for the low value and less than $\delta$ for the high value. For example, the low value can be $\lfloor \frac{k}{s} n \rfloor$ if $\delta << 0.5$, and the high value can be a Hoeffding bound \cite{hoeffding63}:
$$\min(\lceil n \left(\frac{k}{s} + \sqrt{\frac{\ln \frac{1}{\delta}}{2 s}}\right) \rceil, n).$$

\begin{thm} \label{binsearch_thm}
Together, the following conditions ensure that the $\hat{m}$ computed by binary search is within one of $m_u$:
\begin{itemize}
\item $\forall i > 0: \hat{L}(n, m_u - i, s, k) > \delta$
\item $\forall i > 1: \hat{L}(n, m_u + i, s, k) < \delta$
\item $\hat{L}(n, m_u - 1, s, k) \geq \hat{L}(n, m_u , s, k) \geq \hat{L}(n, m_u + 1, s, k)$
\end{itemize}
\end{thm}

\begin{proof}
The first two conditions imply that binary search returns as $\hat{m}$ neither $m_u - 2$ or less nor $m_u + 2$ or more, since it makes the condition $\hat{L}(n, m, s, k) \geq \delta$ and $\hat{L}(n, m + 1, s, k) < \delta$ impossible for those $m$ values. Adding the third condition ensures that there is a unique value $\hat{m}$ such that $\forall m \leq \hat{m}: \hat{L}(n, m, s, k) \geq \delta$ and $\forall m \geq \hat{m}: \hat{L}(n, m, s, k) < \delta$. This ensures that the algorithm progresses and returns that $\hat{m}$ value.
\end{proof}

The first two conditions in Theorem \ref{binsearch_thm} imply that estimates of left tails need not be very accurate for $m$ values far from $m_u$. Let 
$$\Delta(m) = L(n, m, s, k) - L(n, m + 1, s, k).$$
The third condition is met if estimates have more accuracy than the gaps between left tails for $m$ values near $m_u$:
$$ \forall m \in \sset{m_u - 1, m_u, m_u + 1}: $$
$$\hat{L}(n, m,  s, k) > L(n, m, s, k) - \frac{\Delta(m - 1)}{2},$$
and 
$$ \hat{L}(n, m,  s, k) < L(n, m, s, k) + \frac{\Delta(m)}{2}. $$
The next theorem indicates the sizes of these gaps.

\begin{thm} \label{gap_thm}
$$ \Delta(m) = p(n, m, s, k) \frac{s - k}{n - m}.$$
\end{thm}

\begin{proof}
The difference between left tails for $m$ and for $m + 1$ is the probability that converting a random "failure" in the population into a "success" causes a sample with $k$ or fewer successes to have $k + 1$ or more. This requires the sample to have $k$ successes before the conversion, which has probability $p(n, m, s, k)$, and, given that, the converted element must be a sample, which has probability $\frac{s - k}{n - m}$.
\end{proof}

The left tail $L(n, m, s, k)$ has $k + 1$ terms, and around $m_u$ they sum to approximately $\delta$, with the rightmost term $p(n, m, s, k)$ the largest. So 
$$p(n, m, s, k) \approx \frac{\delta}{k + 1}$$ 
is a conservative estimate. That gives
$$\Delta(m) \approx \frac{\delta}{k + 1} \frac{s - k}{n - m} \geq \frac{\delta}{nk}.$$
So left tail estimates with error less than $\frac{\delta}{nk}$ should yield $\hat{m}$ within one of $m_u$. For example, with $\delta = 0.01$, $k$ ten million, and $n$ one trillion, we would like about $2 + 7 + 12 = 21$ digits of accuracy. For reference, most platforms map Python floating point numbers to the IEEE-754 doubles, which have about 16 digits (53 bits) of precision. 

\section{Methods of Computation}
To compute $m_u$ within one, we will need to do some combination of:
\begin{itemize}
\item Limit $k$, $n$, and $\frac{1}{\delta}$.
\item Use higher-precision arithmetic than for standard doubles.
\item Use numerical methods that avoid loss of precision.
\end{itemize}

On the first point, for big data, population sizes $n$ can be in the billions or on the order of a trillion, and sample sizes $s$ tend to be in the millions, so $k$ can be of the same order. Often, for 95\% confidence, $\delta = 5\%$, but $\delta$ values can be one or two orders of magnitude smaller, to achieve higher confidence or to have reasonable confidence in simultaneous estimates of multiple frequencies. Packages such as Python's stat.hypergeom offer about 5 digits of accuracy in estimating the cdf of the hypergeometric distribution -- not enough to differentiate between left tails that differ by $\frac{\delta}{nk}$ for $n$ more than a million. So it is important to explore the other two points.

On the second point, Python offers a Decimal class that allows programmers to select the level of precision for arithmetic operations. It is easy to use, and using it does not cause infeasible slowing, at least not for precision up to a few hundred digits. For the computation methods that we describe next, setting precision to 30 digits gives sufficient accuracy to produce $\hat{m}$ within one of $m_u$ for $\delta = 0.05$, $n$ one trillion, and $k$ nine million. 

On the third point, we have had success with two different methods of computing the left tail: one is computing terms $p(n, m, s, j)$for $0 \leq j \leq k$ using combinatorial identities and ordering computations to reduce loss of precision, and the other is estimating terms using Stirling's approximation. In both cases, we ignore very small tail terms, which may increase error but reduces computation. 

For the first method, recall that:
$$ p(n, m, s, j) \equiv {{n}\choose{s}}^{-1} {{m}\choose{j}} {{n - m}\choose{s - j}}. $$
Define
$$ T(h, j) \equiv \prod_{i = 0}^{j - 1} (h - i). $$
Since 
$$ T(h, j) = \frac{h!}{(h - j)!},$$
$$ {{h}\choose{j}} = \frac{T(h, j)}{j!}.$$
Then
$$ p(n, m, s, j) = \frac{T(m, j) T(n - m, s - j) s!}{j! (s - j)! T(n, s)} $$
$$ = \frac{T(m, j) T(n - m, s - j)}{T(n, s)} {{s}\choose{j}}$$
$$ = \frac{T(m, j) T(n - m, s - j) T(s, j)}{j! T(n, s)}.$$
The numerator and denominator each have $s + j$ terms, with some terms as large as $n - m$ or $n$, making both huge. To avoid creating huge numbers or floating point underflow/overflow problems, use Loader's \cite{loader10} method:
\begin{itemize}
\item Start with a list (or iterator) of numerator terms, and one of denominator terms.
\item Assign $v = 1$.
\item If $v < 1$ and there are more numerator terms, multiply $v$ by one of them, and remove it.
\item If $v > 1$ and there are more denominator terms, divide $v$ by one of them, and remove it.
\item If there are more terms in either list, repeat.
\item Return $v$.
\end{itemize}

Use that method for the largest term in the tail. For smaller terms, note that
$$p(n, m, s, j + 1) / p(n, m, s, j) = \frac{(m - j)(s - j)}{(j + 1)(n - m - s + j + 1)},$$
and multiply or divide by that ratio to compute successive terms. This reduces computation.

For the second method, use a version of Stirling's approximation:
$$ \ln n! \approx n \ln n - n + \frac{1}{2} \ln(2 \pi n) + \frac{1}{12n} - \frac{1}{360n^3} + \frac{1}{1260 n^5} - \frac{1}{1680n^7}. $$
(Compute the terms in reverse order to avoid losing the smaller terms to roundoff.) Let $A(n)$ be the RHS. Then
$$ \ln T(h, j) = A(h) - A(h - j), $$
so
$$ \ln p(n, m, s, j) = $$
$$ A(m) - A(m - j) + A(n - m) - A(n - m - s + j) + A(s) - A(s - j)$$
$$ - A(j) - A(n) + A(n - s), $$
and we can compute the $p(n, m, s, j)$ by exponentiating the RHS. As before, do this for the largest term. For the other terms, use
$$ \ln p(n, m, s, j + 1) - \ln p(n, m, s, j) =$$
$$  \ln(m - j) + \ln(s - j) - \ln(j + 1) - \ln(n - m - s + j + 1). $$

Both the direct computation method and the Stirling's approximation method have been tested for a variety of inputs. For $\delta = 0.05$, $n$ one trillion, $s$ ten million, and $k$ nine million, the gaps between tails for $m$ values near $m_u$ are on the order of one in three billion. Both methods return $\hat{m} = 900\,156\,008\,220$ and compute the corresponding lower bound to be $899\,843\,820\,749$. This is for normal Python floats for direct computation (no need for Decimal with higher precision) and with precision set to 30 for the Stirling's approximation method.

For the direct computation method, it is O($s$) to compute the largest term of the tail, O($1$) for each of the other $k$ terms, and there are O($\lg n$) tail computations for the binary search. So the entire computation uses O($s \lg n$) time. Using Stirling's approximation, each of the $k$ tail terms, including the largest, requires O($1$) time. So the overall computation requires O($k \lg n$) time. For $n$ one trillion, neither computation is instant; they require a few minutes on an old iMac. 


\begin{thebibliography}{00}
\bibitem{fan62} C. T. Fan, M. E. Muller, and I. Rezucha, ``Development of Sampling Plans by Using Sequential (Item by item) Selection Techniques and Digital Computers," Journal of the American Statistical Association, vol. 57, pp. 387--402, 1962. 
\bibitem{vitter84} J. S. Vitter, ``Faster Methods of Random Sampling," Comm. of the ACM, vol. 27, no. 7, pp. 703--718, 1984.
\bibitem{vitter85} J. S. Vitter, ``Random Sampling with a Reservoir," ACM Trans. on Mathematical Software, vol. 11, no. 1, pp. 37--57, 1985.
\bibitem{olken86} F. Olken and D. Rotem, ``Simple random sampling from relational databases," Proceedings of the 12th International Conference on Very Large Data Bases, VLDB '86, pp. 160-169, 1986.
\bibitem{olken90} F. Olken and D. Rotem. ``Random sampling from database files: A survey, " Statistical and Scientific Database Management, 5th International Conference, SSDBM, pp. 92--111, 1990.
\bibitem{olken93} F. Olken, ``Random Sampling from Databases," PhD Thesis, University of California at Berkeley, 1993.
\bibitem{acharya00} S. Acharya, P. B. Gibbons, and V. Poosala, ``Congressional samples for approximate answering of group-by queries," SIGMOD Rec., vol. 29, no. 2, pp. 487--498, 2000.
\bibitem{haas96a} P. J. Haas, ``Hoeffding inequalities for join-selectivity estimation and online aggregation," IBM Research Report RJ 10040, IBM Almaden Research Center, San Jose, CA, 1996.
\bibitem{haas96b} P. J. Haas, ``Large-sample end deterministic confidence intervals for online aggregation," IBM Research Report RJ 10050, IBM Almaden Research Center, San Jose, CA, 1996.
\bibitem{hellerstein97} J. M. Hellerstein, P. J. Haas, and H. J. Wang, ``Online aggregation," SIGMOD Rec., vol. 26, no. 2, pp. 171-182, 1997.
\bibitem{denning80} D. E. Denning, ``Secure Statistical Databases with Random Sample Queries," ACM Trans. Database Syst., vol. 5, no. 8, pp. 291--315, 1980.
\bibitem{hoeffding63} W. Hoeffding, ``Probability inequalities for sums of bounded random variables," Journal of the American Statistical Association, vol. 58, no. 301, pp. 13--30, 1963.
\bibitem{chvatal79} V. Chv\'atal, ``The tail of the hypergeometric distribution," Discrete Mathematics, vol. 25, no. 3, pp. 285--287, 1979.
\bibitem{loader10} C. Loader, ``Fast and accurate computation of binomial probabilities," 2000. 
\end{thebibliography}
\end{document}